\documentclass[journal,twoside,web]{ieeecolor}
\usepackage{lcsys}

\usepackage{cite}
\usepackage{amsmath,amssymb,amsfonts}
\usepackage{algorithmic}
\usepackage{graphicx}
\usepackage{textcomp}

\usepackage{epsfig}

\usepackage{cite}
\usepackage{xcolor}
\usepackage{xurl}
\usepackage{hyperref}
\hypersetup{
    colorlinks,
    linkcolor={red!50!black},
    citecolor={blue!50!black},
    urlcolor={blue!80!black}
}
\usepackage{bm}
\usepackage{breqn}

\usepackage{amsthm}

\newcommand{\reals}{\mathbb{R}}
\newcommand{\R}{\reals}
\newcommand{\Rnonneg}{\reals_{\geq 0}}

\renewcommand{\S}{\mathbb{S}}

\newcommand{\sign}{\operatorname{sign}}

\newcommand{\pd}{\S_{++}}

\newcommand{\Dcal}{\mathcal{D}}

\newcommand{\Ical}{\mathcal{I}}

\newcommand{\Pcal}{\mathcal{P}}
\newcommand{\Scal}{\mathcal{S}}
\newcommand{\Ucal}{\mathcal{U}}
\newcommand{\Xcal}{\mathcal{X}}
\newcommand{\Ycal}{\mathcal{Y}}

\newcommand{\classK}{class~$\mathcal{K}$ }

\newcommand{\classKL}{class~$\mathcal{KL}$ }

\newcommand{\eqn}[1]{\begin{align} #1 \end{align}}
\newcommand{\eqnN}[1]{\begin{align*} #1 \end{align*}}

\newcommand{\bmat}[1]{\begin{bmatrix}#1\end{bmatrix}}

\newcommand{\norm}[1]{\left\Vert #1 \right \Vert}
\newcommand{\abs}[1]{\left | #1 \right |}

\newcommand{\argmin}[1]{\underset{#1}{\text{argmin}}}

\theoremstyle{plain}
\newtheorem{theorem}{Theorem}

\theoremstyle{definition}
\newtheorem{definition}{Definition}
\newtheorem{assumption}{Assumption}
\newtheorem{problem}{Problem}

\theoremstyle{remark}
\newtheorem{remark}{Remark}

\usepackage{setspace}

\begin{document}
\title{Safe and Robust Observer-Controller Synthesis using Control Barrier Functions}
\author{Devansh R. Agrawal, \IEEEmembership{Graduate Student Member, IEEE}, Dimitra Panagou, \IEEEmembership{Senior Member, IEEE}
\thanks{The authors would like to acknowledge the support of the National Science Foundation (NSF) under grant no. 1942907. Both authors are at the Aerospace Engineering department, of the University of Michigan, Ann Arbor, USA. {\tt \{devansh, dpanagou\}@umich.edu}}
}

\maketitle

\begin{abstract}
This paper addresses the synthesis of safety-critical controllers using estimate feedback. We propose an observer-controller interconnection to ensure that the nonlinear system remains safe despite bounded disturbances on the system dynamics and measurements that correspond to partial state information. The co-design of observers and controllers is critical, since even in undisturbed cases, observers and controllers designed independently may not render the system safe. 

We propose two approaches to synthesize observer-controller interconnections. The first approach utilizes Input-to-State Stable observers, and the second uses Bounded Error observers. Using these stability and boundedness properties of the observation error, we construct novel Control Barrier Functions that impose inequality constraints on the control inputs which, when satisfied, certifies safety. We propose quadratic program-based controllers to satisfy these constraints, and prove Lipschitz continuity of the derived controllers.  Simulations and experiments on a quadrotor demonstrate the efficacy of the proposed methods.\looseness=-1
\end{abstract}


\section{Introduction}
\label{sec:introduction}
\IEEEPARstart{F}{or} safety-critical systems, one must not only design controllers that prioritize system safety above all else, but also certify that the system will remain safe when deployed. In recent years, Control Barrier Functions~(CBFs)~\cite{ames2016control} have become a popular method to design safety-critical controllers, since a certifiably safe control input can be computed efficiently for nonlinear systems. Many extensions have been proposed to address specific challenges in using CBFs, including robustness~\cite{jankovic2018robust, alan2021safe}, sampled-data considerations~\cite{breeden2021control} and integration with high-level planners~\cite{agrawal2021constructive}. However, these works assume the controller has access to perfect state information. In most practical systems, the true state of the system is unknown and must be reconstructed using only (often noisy) measurements obtained from sensors. In such systems, it is common to design a full-state feedback controller, and then replace the state by an estimate provided by an observer~\cite[Sec. 8.7]{bernard2022observer}. It is well established that a controller capable of stabilizing a system with perfect state information may fail to do so when using the state estimate~\cite[Ch. 12]{khalil2015nonlinear}. Similarly, the use of imperfect information for feedback control may cause safety violations. 

In this paper, we study the implications on safety that arise due to imperfect and partially available information, and propose a method to design safe observer-controllers. This important challenge has only recently received some attention. Measurement-Robust CBFs~\cite{dean2021guaranteeing} have been proposed to address control synthesis in output-feedback, in the context of vision-based control. The authors assume sensors are noiseless and an imperfect inverse of the measurement map is known, i.e. from a single measurement, a ball containing the true state is known. Using this bound, a second-order cone program-based controller was proposed, although the Lipschitz continuity of this controller is yet to be established~\cite{dean2021guaranteeing}. For many safety-critical systems, the measurement maps are non-invertible, limiting the scope for this method.

In~\cite{clark2019control}, a safety critical controller is proposed for stochastic systems, and a probabilistic safety guarantee is proved. The authors consider linear (non-invertible) measurement maps, additive gaussian disturbances, and specifically use the Extended Kalman Filter~(EKF) as the observer. In~\cite{jahanshahi2020synthesis} this work is extended to consider a broader class of control-affine systems, and probabilistic guarantees of safety over a finite forward interval are obtained. Establishing safety in a deterministic (non-probabilistic) sense or using alternative observers remains challenging. It has also been demonstrated that in some cases, safety guarantees can be obtained by modeling the system as a Partially Observable Markov Decision Process, e.g.~\cite{ahmadi2019safe}, although such methods are computationally expensive for high-dimensional systems and are more suitable for systems with discrete action/state spaces.

The primary contribution of this paper is in synthesizing safe and robust interconnected observer-controllers in such a manner as to establish rigorous guarantees of safety, despite bounded disturbances on the system dynamics and sensor measurements. We propose two approaches to solve this problem, owing to the wide range of nonlinear observers~\cite{bernard2022observer}. The first approach utilizes the class of Input-to-State Stable observers~\cite{shim2015nonlinear}. The second approach employs the more general class of `Bounded Error' observers, in which a set containing the state estimation error is known at all times. This class of observers includes the Deterministic Extended Kalman Filter (DEKF)~\cite[Ch. 11.2]{khalil2015nonlinear}, Lyapunov-based sum-of-squares polynomial observers~\cite{PylorofObserver}, and others discussed later. We show that our safe estimate-feedback controller can be obtained by solving quadratic programs (QP), and prove Lipschitz continuity of these controllers, allowing for low-computational complexity real-time implementation. The efficacy of the methods is demonstrated both in simulations and in experiments on a quadrotor.

\section{Preliminaries and Background}

\subsubsection*{Notation} Let $\R$ be the set of reals, $\Rnonneg$ the set of non-negative reals and $\pd^n$ the set of symmetric positive definite matrices in $\R^{n\times n}$. $\lambda_{min}(P), \lambda_{max}(P)$ denote the smallest and largest eigenvalues of $P \in \pd^n$.  For $x \in \R^n$, $x_i$ is the $i$-th element, $\norm{x}$ is the Euclidean norm. The norm of a signal $w : \Rnonneg \to \R^q$ is $\norm{w(t)}_\infty \triangleq \sup_{t \geq 0} \norm{w(t)}$. $\gamma_f$ denotes the Lipschitz constant of a Lipschitz-continuous function $f : \R^n \to \R^m$. Class~$\mathcal{K}$, extended \classK  and \classKL functions are as defined in \cite{XU201554}. Lie derivatives of a scalar function $h : \Xcal \to \R$, ($\Xcal \subset \R^n$), along a vector field $f : \Xcal \to \R^n$ are denoted $L_fh(x) = \frac{\partial h}{\partial x}(x) f(x)$. If vector fields has an additional dependency, e.g., $f: \Xcal \times \R^p \to \R^n$, the notation $L_fh(x, y) = \frac{\partial h}{\partial x}(x) f(x, y)$ is used.

\subsubsection{System}

Consider a nonlinear control-affine system:
\begin{subequations}
\label{eqn:sys}
\eqn{
\dot x &= f(x) + g(x) u + g_d(x) d(t), \label{eqn:main_dyn}\\
y &= c(x) + c_d(x) v(t), \label{eqn:measurements}
}
\end{subequations}
where $x \in \Xcal \subset \R^n$ is the system state, $u \in \Ucal \subset \R^m$ is the control input, $y \in \R^{n_y}$ is the measured output, $d: \Rnonneg \to \R^{n_d}$ is a disturbance on the system dynamics, and $v: \Rnonneg \to \R^{n_v}$ is the measurement disturbance. We assume $d$ and $v$ are piecewise continuous, bounded disturbances, $\sup_{t}\norm{d(t)}_\infty = \bar d$ and $\norm{v(t)}_\infty \leq \bar v$ for some known $\bar d, \bar v < \infty$. The functions $f : \Xcal \to \R^n$, $g: \Xcal \to \R^{n \times m}$, $c : \Xcal \to \R^{n_y}$, $g_d : \Xcal \to \R^{n \times n_d}$, and $c_d : \Xcal \to \R^{{n_y} \times n_v}$ are all assumed to be locally Lipschitz continuous.  Notice that $g_d(x) d(t)$ accounts for either matched or unmatched disturbances.

In observer-controller interconnections, the observer maintains a state estimate $\hat x \in \Xcal$, from which the controller determines the control input. The observer-controller interconnection is defined to be of the form:
\begin{subequations}
\label{eqn:observer-controller}
\eqn{
\dot {\hat x} &= p(\hat x, y) + q(\hat x, y) u, \label{eqn:observer}\\
u &= \pi(t, \hat x, y), \label{eqn:controller}
}
\end{subequations}
where $p : \Xcal \times \R^{n_y} \to  \R^n$, $q: \Xcal \times \R^{n_y} \to \R^{n \times m}$ are locally Lipschitz in both arguments. The feedback controller $\pi : \Rnonneg \times \Xcal \times \R^p \to \Ucal$ is assumed piecewise-continuous in $t$ and Lipschitz continuous in the other two arguments. Then, the closed-loop system formed by~(\ref{eqn:sys}, \ref{eqn:observer-controller}) is
\begin{subequations}
\label{eqn:closed_loop_sys}
\eqn{
\dot x &= f(x) + g(x)u + g_d(x) d(t),\\
\dot {\hat x} &= p(\hat x, y) + q(\hat x, y) u,\\
x(0) &= x_0, \ \hat x(0) = \hat x_0,
}
\end{subequations}
where $y$ and $u$ are defined in \eqref{eqn:measurements} and \eqref{eqn:controller} respectively. Under the stated assumptions, there exists an interval $\Ical = \Ical(x_0, \hat x_0) = [0, t_{max}(x_0, \hat x_0))$ over which solutions to the closed-loop system exist and are unique~\cite[Thm 3.1]{KhalilNonlinearSys}.


\subsubsection{Safety}

Safety is defined as the true state of the system remaining within a \emph{safe set}, $\Scal \subset \Xcal$, for all times $t \in \Ical$. The safe set $\Scal$ is defined as the super-level set of a continuously-differentiable function $h: \Xcal \to \R$:
\eqn{
\label{eqn:scal}
\Scal = \{x \in \Xcal : h(x) \geq 0\}.
}

A state-feedback controller\footnote{In \emph{state-feedback} the control input is determined from the true state, $u = \pi(t, x)$. In \emph{estimate-feedback} the input is determined from the state estimate and measurements, $u = \pi(t, \hat x, y)$.} $\pi : \Rnonneg \times \Xcal \to \Ucal$ \emph{renders system~\eqref{eqn:sys} safe} with respect to the set $\Scal$, if for the closed-loop dynamics $\dot x = f(x) + g(x) \pi(t, x) + g_d(x) d(t)$, the set $\Scal$ is \emph{forward invariant}, i.e., $x(0) \in \Scal \implies x(t) \in \Scal \ \forall t \in \Ical$. In output-feedback we define safety as follows:
\begin{definition}
An observer-controller pair~\eqref{eqn:observer-controller} \emph{renders system~\eqref{eqn:sys} safe} with respect to a set $\Scal \subset \Xcal$ from the initial-condition sets $\Xcal_0, \hat \Xcal_0 \subset \Scal$ if for the closed-loop system~\eqref{eqn:closed_loop_sys},
\eqn{
x(0) \in \Xcal_0 \text{ and } \hat x(0) \in \hat \Xcal_0 \implies x(t) \in \Scal \quad \forall t \in \Ical.
}
\end{definition}
Note the importance of the observer-controller connection, i.e., using only $\hat x(t)$, we must obtain guarantees on $x(t)$. 

\subsubsection{Control Barrier Functions}

Control Barrier Functions (CBFs) have emerged as a tool to characterize and find controllers that can render a system safe~\cite{ames2016control}. Robust-CBFs~\cite{jankovic2018robust} also account for the disturbances $d(t)$ in~\eqref{eqn:main_dyn}. We introduce a modification to reduce conservatism, inspired by~\cite{alan2021safe}.
\begin{definition}
A continuously differentiable function $h : \Xcal \to \R$ is a \emph{Tunable Robust CBF} (TRCBF) for  system~\eqref{eqn:sys} if there exists a \classK function $\alpha$, and a continuous, non-increasing function $\kappa: \Rnonneg \to \R$ with $\kappa(0) = 1$, s.t.
\eqn{
\sup_{u \in \Ucal} & \ L_fh(x) +  L_gh(x) u + \alpha(h(x)) \notag \\ 
&  \ \geq \kappa(h(x)) \norm{L_{g_d}h(x)} \bar d,  \ \forall x \in \Scal.
} 
\end{definition}
Examples include $\kappa(r) = 1$ and $\kappa(r) = 2/(1 + \exp(r))$.
Given a TRCBF $h$ for~\eqref{eqn:sys}, the set of safe control inputs is
\eqn{
\label{eqn:Ktrcbf}
K_{trcbf}(x) = \{ &u \in \Ucal :  \ L_fh(x) + L_gh(x) u -\notag\\ &\quad \kappa(h(x)) \norm{L_{g_d}h(x)} \bar d \geq -\alpha(h(x))\},
}
and a safe state-feedback controller is obtained by solving a QP, as in~\cite[Eq. 30]{jankovic2018robust}. The main question is:
\begin{problem}
Given a system~\eqref{eqn:sys} with disturbances of known bounds $\norm{d(t)}_\infty \leq \bar d$, $\norm{v(t)}_\infty \leq \bar v$, and a safe set $\Scal$ defined by~\eqref{eqn:scal}, synthesize an interconnected observer-controller~\eqref{eqn:observer-controller} and the initial condition sets $\Xcal_0, \hat \Xcal_0$ to render the system safe.
\end{problem}

We study systems subject to disturbances with a known bound. We will use this bound to derive sufficient conditions on the control policy to guarantee safety satisfaction. In practice, a conservative upper bounds can be used, although future work will address the probabilistic safety guarantees that are possible under probabilistic disturbances.

\section{Main Results}
\subsection{Approach 1}

Approach 1 relies on defining a set of state estimates, $\hat \Scal \subset \Xcal$, such that if the estimate $\hat x$ lies in $\hat \Scal$, the true state $x$ lies in the safe set $\Scal$. The controller is designed to ensure $\hat x \in \hat \Scal$ at all times. We consider Input-to-State Stable observers:
\begin{definition} [Adapted from \cite{shim2015nonlinear}]
An observer~\eqref{eqn:observer-controller} is an \emph{Input-to-State Stable (ISS) Observer} for system~\eqref{eqn:sys}, if there exists a \classKL function $\beta$ continuously differentiable wrt to the second argument, and a \classK function $\eta$ such that 
\eqn{
\label{eqn:iss_bound}
\norm{x(t) - \hat x(t)} \leq \beta(\norm{x(0)-\hat x(0)}, t) + \eta(\bar w), \forall t \in \Ical,}
where $\bar w = \max ( \bar d, \bar v)$. 
\end{definition}
Various methods to design ISS observers for nonlinear systems have been developed, and reader is referred to \cite{shim2015nonlinear, howell2002nonlinear, alessandri2004observer, arcak2001nonlinear, bernard2022observer} and references within for specific techniques. 

The key property of an ISS observer is that the estimation error is bounded with a known bound: for any $\delta > 0$, there exists a continuously differentiable, non-increasing function $M_\delta: \Rnonneg \to \Rnonneg$, such that 
\eqn{
\label{eqn:iss_bound_M}
\norm{x(0) - \hat x(0)} \leq \delta \Rightarrow \norm{x(t) - \hat x(t)} \leq M_\delta(t) \ \forall t \in \Ical.
}
Comparing \eqref{eqn:iss_bound} and \eqref{eqn:iss_bound_M}, $M_\delta(t) = \beta(\delta, t) + \eta(\bar w)$. Define
\eqn{
\hat{\Scal} = \{ \hat x \in \Xcal : h(\hat x) - \gamma_h M_\delta(t) \geq 0 \},
}
the set of safe state-estimates,  and we obtain the property  $\hat x(t) \in \hat \Scal \implies x(t) \in \Scal$ by the Lipschitz continuity of $h$.\footnote{By Lipschitz continuity, $\abs{h(x) - h(\hat x)} \leq \gamma_h \norm{x-\hat x} \implies h(\hat x) - \gamma_h \norm{x - \hat x}  \leq h(x)$. Therefore, if $\hat x \in \hat \Scal$, then $0 \leq h(\hat x) - \gamma_h M_\delta(t) \leq h(\hat x) - \gamma_h \norm{x - \hat x} \leq h(x)$, i.e., $x \in \Scal$. Thus, $\hat x \in \hat \Scal \implies x \in \Scal$.} Then the conditions to guarantee safety are as follows:
\begin{definition}
A continuously differentiable function $h: \Xcal \to \R$ is an \emph{Observer-Robust CBF} for system~\eqref{eqn:sys} with an ISS observer~\eqref{eqn:observer} of known estimation error bound~\eqref{eqn:iss_bound_M}, if there exists an extended \classK function $\alpha$ s.t.\footnote{Recall the notation $L_ph(\hat x, y) = \frac{\partial h}{\partial x} (\hat x) p(\hat x, y)$.}
\eqn{
\sup_{u \in \Ucal} & \ L_ph(\hat x, y) + L_qh(\hat x, y) u \geq - \alpha(h(\hat x) - \gamma_h M_\delta(0))
}
for all $\hat x \in \Scal$, and all $y \in \Ycal(\hat x) = \{ y : y = c(x) + c_d(x) v(t)  \mid \norm{x - \hat x} \leq M_\delta(0), \norm{v} \leq \bar v\}$, an overapproximation of the set of possible outputs.\footnote{$\Ycal$ is defined using $M_\delta(0)$ instead of $\delta$ since $\Ycal(\hat x(t))$ must contain the set of possible outputs at time $t$ for all $t \in \Ical$.}
\end{definition}

\begin{theorem}
\label{thm:margins}
For system~\eqref{eqn:sys}, suppose the observer~\eqref{eqn:observer} is ISS with estimation error bound~\eqref{eqn:iss_bound_M}. Suppose $\Scal$ is defined by an Observer-Robust CBF $h: \Xcal \to \R$ associated with extended \classK function $\alpha$. If the initial conditions satisfy
\eqn{
\hat x(0) &\in \hat \Xcal_0 = \{ \hat x  \in \Scal : h(\hat x) \geq \gamma_h M_\delta(0) \} \label{eqn:initial_xhat0},\\
x(0) & \in \Xcal_0 = \{ x \in \Scal : \norm{x(0) - \hat x(0)} \leq \delta\},
}
then any Lipschitz continuous estimate-feedback controller $u = \pi(t, \hat x, y) \in K_{orcbf}(t, \hat x, y)$ where
\eqn{
\label{eqn:K_orcbf}
K_{orcbf}(t, \hat x, y) &=  \{ u \in \Ucal : L_ph(\hat x, y) + L_qh(\hat x, y) u \geq \notag \\ 
& \quad  - \alpha\left(h(\hat x) - \gamma_h M_\delta(t)\right) + \gamma_h \dot M_\delta(t)\}
}
renders the system safe from the initial-condition sets $\Xcal_0, \hat \Xcal_0$. 
\end{theorem}

\begin{proof}
Consider the function $H(t, \hat x) = h(\hat x) - \gamma_h M_\delta(t)$. By the Lipschitz continuity of $h$, and~\eqref{eqn:iss_bound_M}, $H(t, \hat x) \geq 0 \implies h(x) \geq 0$.  The total derivative of $H$ is 
\eqnN{
\dot H = \frac{\partial H}{\partial t} + \frac{\partial H}{\partial \hat x} \dot {\hat x} = -\gamma_h \dot M_\delta + L_ph(\hat x, y) + L_qh(\hat x, y) u
}
therefore, for any $\pi(t, \hat x, y) \in K_{orcbf}(t, \hat x, y)$ we have $\dot H \geq -\alpha(H)$. Since $H(0, \hat x_0) \geq 0$ (from the initial condition~\eqref{eqn:initial_xhat0}), $H(t, \hat x) \geq 0, \forall t \in \Ical$, completing the proof. 
\end{proof}

\begin{remark}
Under the same assumptions as Theorem~\ref{thm:margins}, if $\Ucal = \R^m$ and a desired control input $\pi_{des}: \Rnonneg \times \Xcal \to \R^m$ is provided, a QP-based safe estimate-feedback controller is
\eqn{
\label{eqn:margins_QP}
&\pi(t, \hat x, y) =\argmin{u \in \R^m} \norm{u - \pi_{des}(t, \hat x)}^2, \ \text{s.t.}\\
&\ L_ph(\hat x, y) + L_qh(\hat x, y) u \geq -\alpha(h(\hat x) - \gamma_h M_\delta(t)) + \gamma_h \dot M_\delta(t) \notag
}
\end{remark}

\begin{remark}
The constraint in \eqref{eqn:margins_QP} does not explicitly depend on the disturbances $d(t)$ and $v(t)$, since the effect of these disturbances is captured by the estimation error bound $M_\delta(t)$. Furthermore, since $\gamma_h \dot M_\delta(t) \leq 0$,\footnote{Since $M_{\delta}(t) = \beta(\delta, t) + \eta(\bar w)$, and $\beta$ is a \classKL function, $\dot M_{\delta}(t) = \partial \beta / \partial t < 0$. Finally since $\gamma_h \in \Rnonneg$ is a Lipschitz constant, $\gamma_h \dot M_\delta(t) \leq 0$.} the constraint~\eqref{eqn:margins_QP} is easier to satisfy for higher convergence rates of the observer. 
\end{remark}

\begin{remark} For a linear \classK function, $\alpha(r) = \gamma_\alpha r$, if $\dot M_\delta \leq - \gamma_\alpha M_\delta(t)$, a sufficient condition for~\eqref{eqn:margins_QP} is
\eqnN{
L_ph(\hat x, y) + L_qh(\hat x, y) u \geq - \gamma_\alpha h(\hat x).
}
which does not depend on the bound $M_\delta(t)$ or Lipschitz constant $\gamma_h$.  In other words, if the observer converges faster than the rate at which the boundary of the safe set is approached, i.e., if $\dot M_\delta \leq -\gamma_\alpha M_\delta$, then a safe control input can be obtained without explicit knowledge of $M_\delta$ or $\gamma_h$. This matches the general principle that for good performance observers should be converge faster than controllers.
\end{remark}

\subsection{Approach 2}

While in Approach 1 we used the stability guarantees of ISS observers to obtain safe controllers, in Approach 2 we consider observers that only guarantee boundedness of the estimation error. First, we define Bounded-Error Observers: 
\begin{definition}
An observer~\eqref{eqn:observer} is a \emph{Bounded-Error (BE) Observer}, if there exists a bounded set $\Dcal(\hat x_0) \subset \Xcal$ and a (potentially) time-varying bounded set $\Pcal(t, \hat x) \subset \Xcal$ s.t. 
\eqn{
\label{eqn:DimpliesP}
x_0 \in \Dcal(\hat x_0) \implies x(t) \in  \Pcal(t, \hat x) \ \forall t \in \Ical.
}
\end{definition}
Figure~\ref{fig:PandDcal} depicts the sets $\Dcal$ and $\Pcal$. Note, ISS observers are a subset of BE observers, using the definitions $\Dcal(\hat x_0) = \{ x : \norm{x-\hat x_0} \leq \delta\}$ and $\Pcal(t, \hat x) = \{ x : \norm{x - \hat x(t)} \leq M_\delta(t) \}$. BE observers are more general than ISS observers in the following ways: (A)~The sets $\Dcal$ and $\Pcal$ do not have to be norm-balls. For example, they could be zonotopes~\cite{alamo2005guaranteed}, intervals~\cite{jaulin2002nonlinear}, or sublevel sets of sum-of-squares polynomials~\cite{alessandri2020lyapunov}. (B)~The shape and size of $\Pcal$ is allowed to change over time.

\begin{figure}[tbp]
   \centering
   \includegraphics[width=0.85\linewidth]{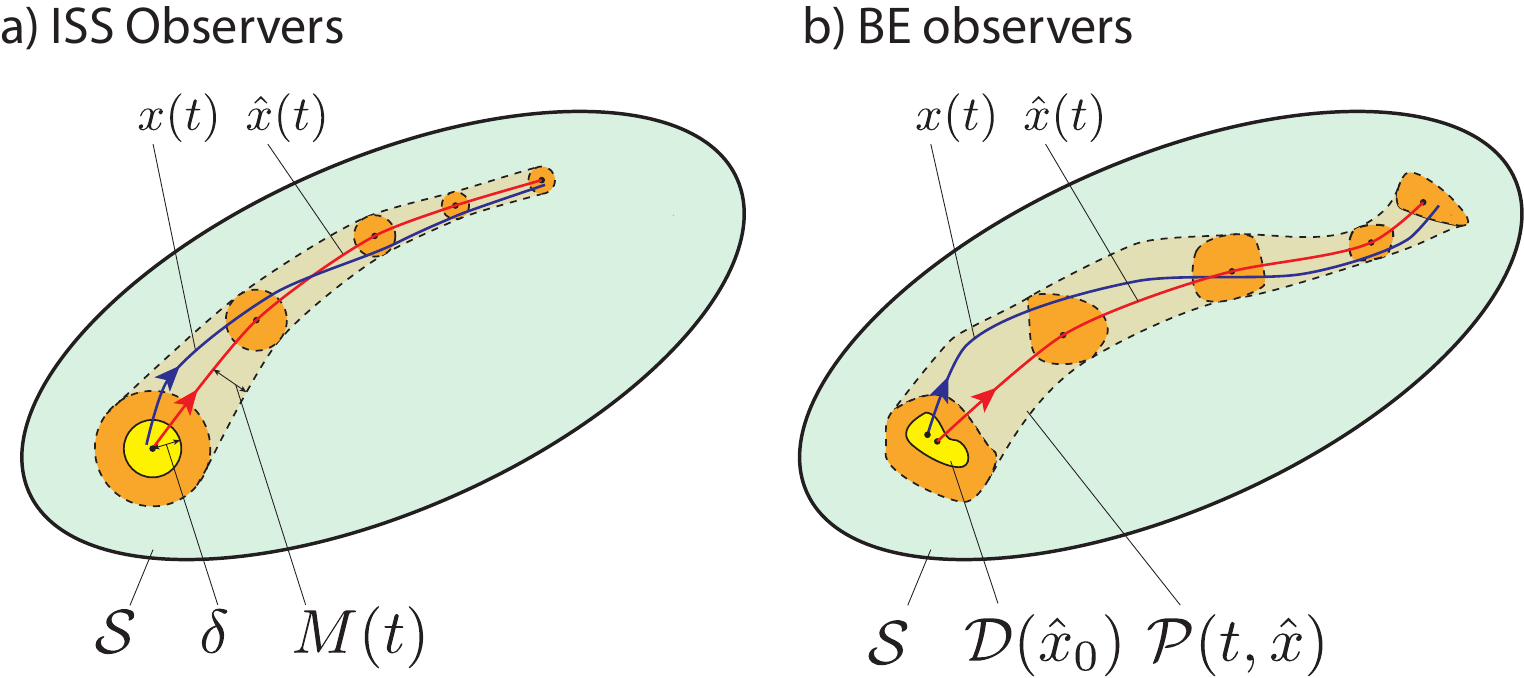} 
   \caption{Depiction of Input-to-State Stable observers and Bounded-Error observers. (a) In ISS observers, the estimation error is bounded by a norm-ball, and must be non-increasing in time. (b) In BE observers, the state estimate must be contained in a bounded set $\Pcal(t, \hat x)$.}
   \label{fig:PandDcal}
\end{figure}

The idea is to find a common, safe input for all $x \in \Pcal(t, \hat x)$:
\begin{theorem}
\label{thm:intersections}
For system~\eqref{eqn:sys}, suppose the observer~\eqref{eqn:observer} is a Bounded-Error observer. Suppose the safe set $\Scal$ is defined by a continuously differentiable function $h : \Xcal \to \R$, where $h$ is a Tunable Robust-CBF for the system. Suppose $\pi: \Rnonneg \times \Xcal \to \Ucal$ is an estimate-feedback controller, piecewise-continuous in the first argument and Lipschitz continuous in the second, s.t.
\eqn{
\label{eqn:intersections_cond}
\pi(t, \hat x) \in  \bigcap_{x \in \Pcal(t, \hat x)} K_{trcbf}(x),
}
where $K_{trcbf}$ is defined in~\eqref{eqn:Ktrcbf}.  Then the observer-controller renders the system safe from the initial-condition sets $x(0) \in \Xcal_0 = \Dcal(\hat x_0)$ and $\hat x_0 \in \hat \Xcal_0 = \{ \hat x: \Pcal(0, \hat x_0) \subset \Scal\}$.
\end{theorem}
\begin{proof}
The total derivative of $h$ for any $x \in \partial \Scal$ and $\pi(t, \hat x) \in K_{trcbf}(x)$ satisfies
\eqnN{
\dot h &= L_fh(x) + L_gh(x) \pi(t, \hat x) + L_{g_d}h(x) w(t)\\
& \geq L_fh(x) + L_gh(x) \pi(t, \hat x) - \kappa(0)\norm{ L_{g_d}h(x)} \bar w\\
& \geq -\alpha(0) = 0
}
since $h(x) = 0$, $\kappa(0) = 1$, and $x(t) \in \Pcal(t, \hat x)$. Therefore, at any $x \in \partial \Scal$, $\dot h \geq 0$, i.e., the system remains safe~\cite{blanchini1999set}.
\end{proof}


In general, designing a controller satisfying~\eqref{eqn:intersections_cond} can be difficult. We propose a method under the following assumptions:\looseness=-1
\begin{assumption}
\label{ass:a}
There exists a known function $a : \Rnonneg \times \Xcal \to \R$, piecewise continuous in the first argument and Lipschitz continuous in the second, such that for all $\hat x \in \Scal$,
\eqnN{
a(t, \hat x) \leq \inf_{x \in \Pcal(t, \hat x)} L_fh(x) - \kappa(h(x))\norm{L_{g_d}h(x)}\bar w + \alpha(h(x)).
}
\end{assumption}
By Assumption 1, $a(t, \hat x)$ lower-bounds the terms in $\dot h$ independent of $u$. These bounds can be obtained using Lipschitz constants. Similarly, we bound each term of $L_gh$:
\begin{assumption}
\label{ass:bi}
There exist known functions $b_i^-, b_i^+: \Rnonneg \times \Xcal \to \R$ for $i = \{ 1, ..., m\}$, piecewise continuous in the first argument and Lipschitz continuous in the second, such that\footnote{Recall, $[L_gh(x)]_i$  refers to the $i$-th element of $L_gh(x)$.}
\eqnN{
b_i^-(t, \hat x) \leq [L_gh(x)]_i \leq b_i^+(t, \hat x)
}
for all $t \geq 0$, all $x \in \Scal$ and all $\hat x \in \{ \hat x : x \in \Pcal(t, \hat x)\}$. Furthermore, suppose $\sign(b_i^-(t, \hat x)) = \sign(b_i^+(t, \hat x))$ at every $t, \hat x \in \Scal$, and that $h$ is of relative-degree 1, i.e., $L_gh(x) \neq 0$. 
\end{assumption}
Intuitively, by assuming $\sign(b_i^-(t, \hat x)) = \sign(b_i^+(t, \hat x))$ it is clear whether a positive or negative $u_i$ increases $\dot h(x, u)$.\footnote{Future work will attempt to relax this assumption. In our limited experience, the estimation error can be sufficiently small that the assumption holds. }
 
\begin{theorem}
\label{thm:intersections_thm}
Consider a system~\eqref{eqn:sys} with $\Ucal = \R^m$ and suppose the observer~\eqref{eqn:observer} is a Bounded-Error observer. Suppose $\Scal$ is the safe set defined by an TRCBF $h$ and Assumptions~\ref{ass:a}, \ref{ass:bi} are satisfied. Suppose $\pi_{des}: \Rnonneg \times \Xcal \to \Ucal$ is a desired controller, piecewise continuous wrt $t$ and Lipschitz continuous wrt $\hat x$. Then the estimate-feedback controller $\pi: \Rnonneg \times \Xcal \to \R^m$
\eqn{
\label{eqn:intersections_pi}
\pi(t, \hat x) &= \argmin{u \in \R^m} \  \norm{u - \pi_{des}(t, \hat x)}^2\\ 
\text{s.t.} \ &a(t, \hat x) + \sum_{i=1}^{m} \min \{b_i^-(t, \hat x) u_i, b_i^+(t, \hat x) u_i\} \geq 0 \notag
}
is piecewise continuous wrt $t$, Lipschitz continuous wrt $x$, and renders the system safe from the initial-condition sets $x_0 \in \Xcal_0 = \Dcal(\hat x_0)$ and $\hat x_0 \in \hat \Xcal_0 = \{ \hat x: \Pcal(0, \hat x_0) \subset \Scal\}$.
\end{theorem}

\begin{proof}
\emph{First, we prove existence and uniqueness of solutions to the QP.} In standard form, the QP~\eqref{eqn:intersections_pi} is equivalent to
\eqn{
\label{eqn:intersections_equiv}
&\min_{u \in \R^m, k \in \R^m} \frac{1}{2} u^T u - \pi_{des}^T u\\
\text{s.t.}&\left[
    \begin{array}{@{} ccc|ccc@{} }b_1^- & ... & 0 & -1 & ...&0\\
b_1^+ & ...& 0  & -1 & ...&0 \\
\vdots&  \ddots & \vdots & \vdots & \ddots &\vdots \\
0 & ... & b_m^- & 0 & ...   & -1 \\
0 & ... & b_m^+  & 0 & ... &  -1\\
\hline
0 & ... & 0 &1  & ... &  1
    \end{array}
    \right]
    \left[
\begin{array}{@{} c@{} }
u_1 \\ \vdots \\u_m \\ \hline k_1  \\ \vdots\\ k_m
\end{array}
\right] \geq
\left[
\begin{array}{@{} c@{} } 0 \\0\\ \vdots\\ 0 \\ 0  \\ \hline -a
\end{array}
\right]
\notag
}
where the dependences on $(t, \hat x)$ were omitted for brevity. Here $k \in \R^m$ is an auxiliary variable encoding the constraint $k_i \leq \min \{ b_i^- u_i, b_i^+ u_i\}$ for all $i = \{ 1, ..., m\}$. This constraint matrix has size $(2m+1, 2m)$. However, since $\sign(b_i^-) = \sign(b_i^+)$ by Assumption~\ref{ass:bi}, only one of either the $(2i-1)$-th or $(2i)$-th constraints can be active.\footnote{Note, if $b_i^- = b_i^+ \neq 0$, then both constraints are equivalent, and thus still means a single constraints is active. Since $L_gh(x) \neq 0$  (Assumption~\ref{ass:bi}),  $b_i^- = b_i^+ \neq 0$ for atleast one of $i =1, ..., m$.}  Considering the sparsity pattern of active constraint matrix, these constraints must be linearly independent. Therefore, the proposed QP has $2m$ decision variables with at most $m+1$ linearly independent constraints, and thus a non-empty set of feasible solutions. Since the cost function is quadratic, there exists a unique minimizer.

\emph{Second, we prove Lipschitz continuity.} Since the active constraints matrix has linearly independent rows, the regularity conditions in~\cite{hager1979lipschitz} are met. Thus the solution $\pi(t, \hat x)$ is Lipschitz continuous wrt $\pi_{des}(t, \hat x), a(t, \hat x), b_i^-(t, \hat x)$ and $b_i^+(t, \hat x)$. Since these quantities are piecewise continuous wrt $t$ and Lipschitz continuous wrt $\hat x$, the same is true for $\pi(t, \hat x)$.

\emph{Finally, we prove safety.} 
Since (omitting $t, x, \hat x$),
\eqnN{
L_gh u = \sum_{i=1}^m [L_gh]_i u_i \geq \sum_{i=1}^m \min \{b_i^- u_i, b_i^+ u_i\},
}
satisfaction of the constraint in~\eqref{eqn:intersections_pi} implies satisfaction of \eqref{eqn:intersections_cond}. Therefore, by Theorem~\ref{thm:intersections}, the system is rendered safe.
\end{proof}

\section{Simulations and Experiments}

Code and videos are available here:{ \small \url{https://github.com/dev10110/robust-safe-observer-controllers} }

\subsubsection{Simulation: Double Integrator}

We simulate a double integrator system without disturbances, to demonstrate the importance of the observer-controller interconnection.  The system is (with $\Ucal = \R$)
\eqn{
\label{eqn:di}
\dot x_1  = x_2, \ \dot x_2 = u, \ y = x_1,}
and the safe set is defined as $\Scal = \{ x : x_1 \leq x_{max} \}$. We use the CBF $h(x) = -x_2 + \alpha_0 (x_{max} - x_1)$. A Luenberger-observer, $\dot{\hat x}  = A \hat x + B u + L(y-C \hat x)$, is used, where $L = 1/2 P^{-1} C^T$ and $P \in \pd^2$ is the solution the Lyapunov equation $PA + A^T P - C^TC = -2 \theta P$ for design parameter $\theta > 0$. This observer is ISS, since for any $\delta > 0$, \eqref{eqn:iss_bound_M} is satisfied with $M_\delta(t) = \sqrt{\lambda_{max}(P) / \lambda_{min}(P)} \delta e^{-\theta t}$. This observer is also a Bounded Error observer since for any $\delta > 0$, \eqref{eqn:DimpliesP} is satisfied with $\Dcal(\hat x_0) = \{ x : \norm{x_0 - \hat x_0} \leq \delta \}$ and $\Pcal(t, \hat x) = \{ x : (x-\hat x)^T P (x - \hat x) \leq \lambda_{max}(P) \delta^2 e^{-2\theta t} \}$. 

We compare the methods proposed in this paper to the CBF-QP of~\cite{ames2016control} (referred to as the Baseline-QP), using $\hat x$ in lieu of $x$. Plots of the resulting trajectory are depicted in Figure~\ref{fig:fig1}, demonstrating safety violation. The trajectory plots under the controllers based on Approaches 1 and 2 are shown in Figure~\ref{fig:motivation}, demonstrating that safety is maintained in both cases. In Approach 2, the function $L_fh(x)$ is affine in $x$ and $L_gh(x) = -1$ is independent of $x$, and therefore the function $a(t, \hat x)$ was determined using a box bound around $\Pcal(t, \hat x)$ and $b_i^-(t, \hat x) = b_i^+(t, \hat x) = -1$. Numerically, we have noticed that for some initial conditions and convergence rates, the controller of Approach 1 is less conservative than the controller of Approach 2, and in other cases the converse is true. Identifying conditions that determine whether Approach 1 or 2 is less conservative remains an open question.

\begin{figure}[tbp]
   \centering
   \includegraphics[width=0.85\linewidth]{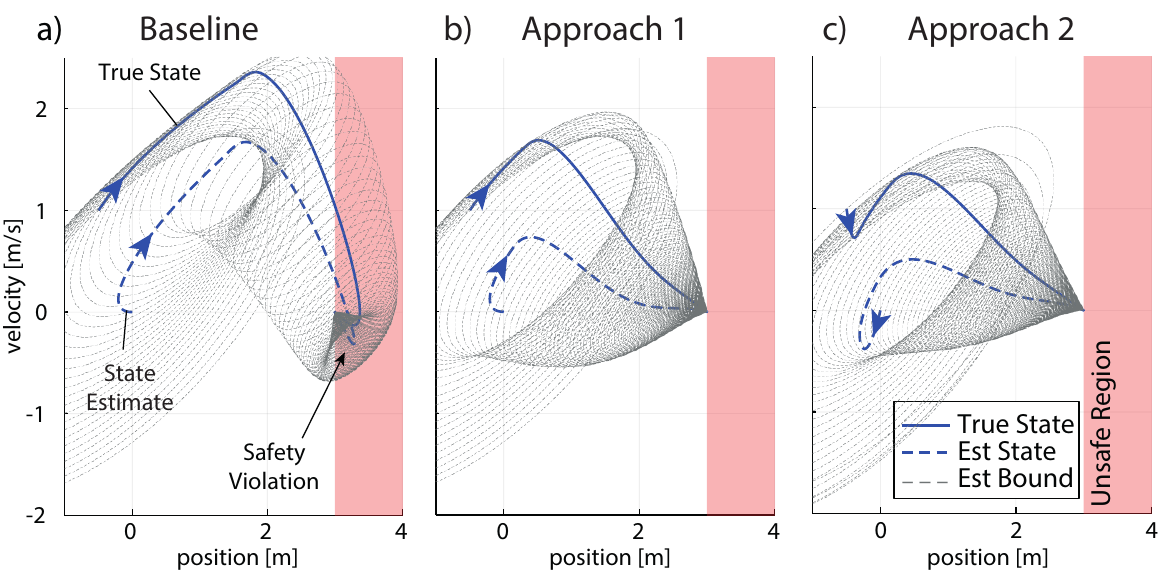} 
   \caption{Simulation results for the Double Integrator~\eqref{eqn:di}, using (a) the baseline CBF controller, (b) Approach 1 and (c) Approach 2. The same initial conditions and observer is used for each simulation. }
   \label{fig:motivation}
   \end{figure}

\subsubsection{Simulation: Planar Quadrotor}

Consider
\eqnN{
\bmat{\dot x_1 \\ \dot x_2 \\ \dot x_3 \\ \dot x_4 \\ \dot x_5 \\ \dot x_6} &= \bmat{x_4 \\ x_5 \\ x_6 \\ 0 \\ - g \\ 0}  + \bmat{ 0 & 0 \\ 0 & 0 \\ 0 & 0 \\ \sin x_3 /m & 0 \\ \cos x_3 / m & 0 \\ 0 & J^{-1} } \bmat{ u_1 \\ u_2} + \bmat{ 0 \\ 0 \\ 0 \\d_1(t) \\ d_2(t) \\ 0} \\
y &= \bmat{x_1, x_2, x_3}^T + \bmat{v_1(t), v_2(t), v_3(t)}^T
}
where $[x_1, x_2]^T$ are the position coordinates of the quadrotor with respect to an inertial coordinate frame, $x_3$ is the pitch angle,  $[x_4, x_5]^T$ are the linear velocities in the inertial frame, and $x_6$ is the rate of change of pitch. The quadrotor has mass~$m=1.0$~kg and moment of inertial~$J=0.25$~kg/m$^2$, and the acceleration due to the gravity is $g = 9.81$~m/s$^{2}$. The control inputs are thrust $u_1$ and torque $u_2$. The disturbances $d : \Rnonneg \to \R^2$ captures the effect of unmodeled aerodynamic forces on the system, bounded by $\norm{d} \leq 2$~m/s$^2$.  The measurement disturbance is $v: \Rnonneg \to \R^3$, bounded by 5~cm for position measurements, and $5^\circ$ for pitch measurements. 

The safety condition is to avoid collision with a circular obstacle at $[x_1^*,  x_2^*]^T$ of radius $r$, i.e., $\Scal = \{ x : (x_1-x_1^*)^2 + (x_2 - x_2^*)^2 - r^2 \geq 0\}$. The CBF proposed in~\cite{wu2016safety} is used. The desired control input is a LQR controller linearized about the hover state. 
The observer is a DEKF adapted from~\cite{reif_EKF}:\footnote{In~\cite{reif_EKF}, only the undisturbed case is demonstrated. The extension to include bounded disturbances can be derived using the same techniques as in the original paper. The additional terms due to the disturbances are bounded using~\cite[Eq. B4]{khalil2015nonlinear}.} Defining constant matrices $D_1 = g_d(x)$ and $D_2 = c_d(x)$, the observer is
\eqnN{
\label{eqn:ekf_observer}
\begin{cases}
\dot {\hat x} &= f(\hat x) + g(\hat x) u + P C^T R^{-1} (y - c(\hat x))\\
\dot P &= PA^T + AP - P C^T R^{-1} C P + Q + 2 \theta P\\
\dot V &= - 2 \theta V + 2\sqrt{V}\left(\norm{D_1^TP^{-1/2}}\bar d + \norm{(LD_2)^T  P^{-1/2}} \bar v\right)
\end{cases}
}
where $\theta \geq 0$ is a design parameter, $A = \frac{\partial }{\partial \hat x} (f(\hat x) + g(\hat x) u)$, $C = \frac{\partial c}{\partial x} (\hat x)$. In the standard form of EKFs~\cite[Sec 5.3]{lewis2017optimal}, the disturbances are assumed to be Weiner processes and  $Q, R$ represent the covariances of the $d(t)$ and $v(t)$. However in the Deterministic EKF, we assume $d(t), v(t)$ are bounded, and thus $Q\in\pd^n, R\in \pd^{n_y}$ can be freely chosen. Assuming there exist positive constants $p_1, p_2$ such that  $p_1 I \leq P(t) \leq p_2 I \ \forall t \in \Ical$, (see \cite[Sec 11.2]{khalil2015nonlinear}), this observer is a Bounded-Error observer, and satisfies~\eqref{eqn:DimpliesP} with $\Dcal(\hat x_0) = \Pcal(0, \hat x_0)$, and $\Pcal(t, \hat x) = \{ x : (x - \hat x)^T P(t)^{-1} (x - \hat x) \leq V(t) \}$.

The method in Approach 2 is used to synthesize the interconnected observer-controller. Specifically, the functions $a$, $b_i^-$, and $b_i^+$ were determined using Lipschitz bounds, and the QP~\eqref{eqn:intersections_pi} is used to determine the control input. 

Figure~\ref{fig:planar_quad} compares the trajectory of the planar quadrotor using the controller proposed in~\cite{wu2016safety} (baseline case) to the proposed controller of Approach 2. In the baseline case, since the state estimate is used in lieu of the true state, safety is violated.  By accounting for the state estimation uncertainty, the proposed controller avoids the obstacle. 

   \begin{figure}[tbp]
   \centering
   \includegraphics[width=0.8\linewidth]{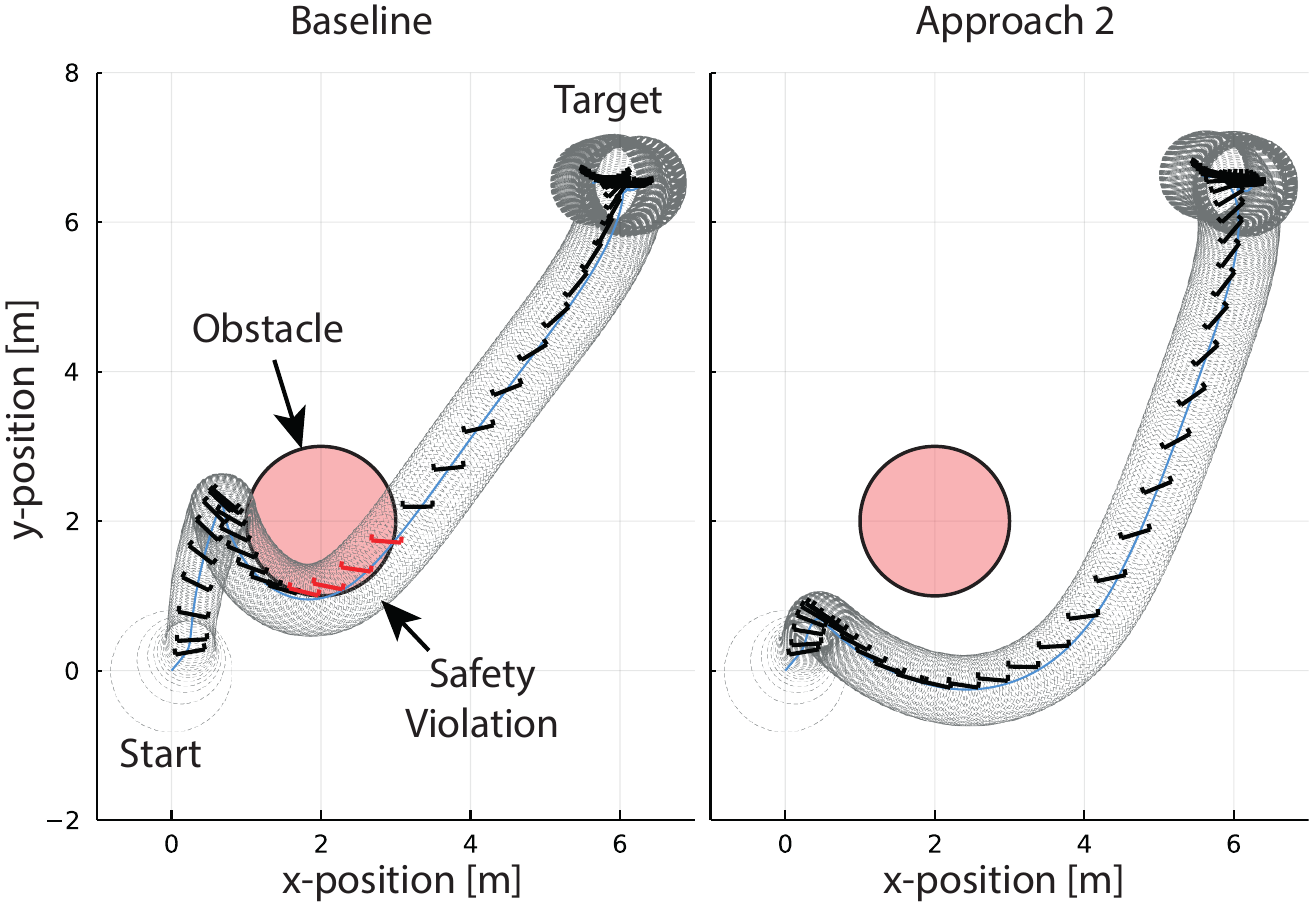} 
   \caption{Simulation Results for the Planar Quadrotor. The objective is to fly the quadrotor from the starting  state to the target position while avoiding the circular obstacle region. The blue lines indicate the path of the state estimate and grey lines the the projection of $\Pcal(t, \hat x)$ on the x-y plane. The icons show the quadrotor's true position every 0.2~s and is colored red while violating safety.  (a) uses the baseline CBF controller, and (b) uses Approach~2. }
   \label{fig:planar_quad}
      \end{figure}

\subsubsection{Experiments: 3D Quadrotor}

For our experiments, we use the Crazyflie~2.0 quadrotor, using the on-board IMU and barometer sensors and an external Vicon motion capture system. The objective was to fly in a figure of eight trajectory, but to not crash into a physical barrier placed at $x=0.5$~meters. State was estimated using an EKF~\cite{MuellerHamerUWB2015},  assuming the true state lies within the 99.8\% confidence interval of the EKF. To design the controller, first $\pi_{des}(t, \hat x)$ is computed using an LQR controller, which computes desired accelerations wrt to an inertial frame to track the desired trajectory. This command is filtered using a safety critical QP, either the baseline CBF-QP (Figure~\ref{fig:fig1}a) or the proposed QP using Approach 2~\eqref{eqn:intersections_pi} (Figure~\ref{fig:fig1}c). Finally, the internal algorithm of the Crazyflie (based on~\cite{mellinger2011minimum}) is used to map the output of the QP to motor PWM signals. The magnitude of the disturbances was estimated by collecting experimental data when the quadrotor was commanded to hover. The trajectories from the two flight controllers are compared in~Figure~\ref{fig:fig1}. In the baseline controller, the quadrotor slows down as it approaches the barrier, but still crashes into barrier. In the proposed controller, the quadrotor remains safe,~Figure~\ref{fig:fig1}e.


   \begin{figure}[tbp]
   \centering
   \includegraphics[width=0.8\linewidth]{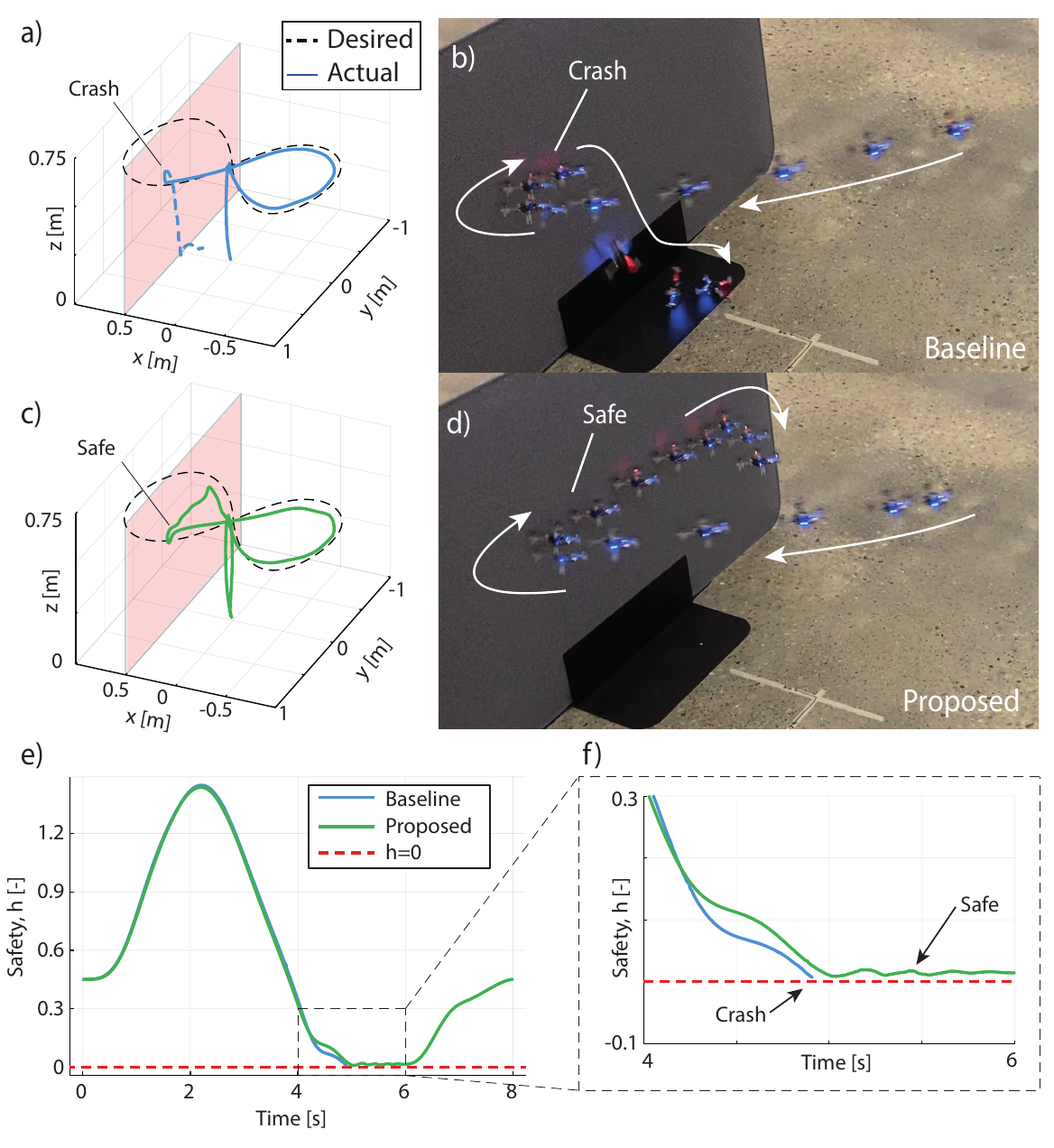} 
   \caption{Experimental results. The quadrotor is commanded to track a figure-of-eight trajectory, while avoiding the physical barrier at $x=0.5$~m. Ground truth trajectories are plotted in (a, c) for the baseline CBF and proposed controllers respectively. Snapshots from the experiment are show in~(b,~d). (e, f) Plots of the safety value, $h$ over time for both trajectories. }
   \label{fig:fig1}
   \vspace{-5mm}
\end{figure}

\section{Conclusion}
In this paper we have developed two methods to synthesize observer-controllers that are robust to bounded disturbances on system dynamics and measurements, and maintain safety in the presence of imperfect information. We have demonstrated the efficacy of these methods in simulation and experiments. Future work will investigate methods to learn the disturbance, such that the controller can adaptively tune itself to achieve better performance, and to extend the work to handle probabilistic guarantees of safety when the system is subject to stochastic disturbances instead of bounded disturbances.

{
\setstretch{0.9}
\bibliographystyle{ieeetr}
\bibliography{biblio}
}
\end{document}